\documentclass[10pt]{article} 
\usepackage[preprint]{tmlr}

\usepackage{amsmath,amsfonts,bm}

\def\eqref#1{equation~\ref{#1}}

\def\1{\bm{1}}

\DeclareMathAlphabet{\mathsfit}{\encodingdefault}{\sfdefault}{m}{sl}
\SetMathAlphabet{\mathsfit}{bold}{\encodingdefault}{\sfdefault}{bx}{n}

\usepackage{hyperref}
\usepackage{url}
\usepackage{booktabs}
\usepackage{multirow}
\usepackage{makecell}
\usepackage{algorithm}
\usepackage{amsmath}
\usepackage{amsthm}
\usepackage{amssymb} 
\usepackage{algpseudocode}
\usepackage{fontawesome}
\usepackage{tikz}
\usetikzlibrary{positioning,arrows.meta,calc,fit,backgrounds}
\usepackage{graphicx}
\usepackage{xcolor}
\newtheorem{definition}{Definition}[section]

\newtheorem{lemma}{Lemma}[section]
\newtheorem{theorem}{Theorem}[section]

\title{SiST-GNN: Simultaneous Spatial-Temporal Message Passing for Dynamic Graph Representation Learning}

\author{\name Shubhajit Roy \email royshubhajit@iitgn.ac.in \\
      \addr Department of Computer Science and Engineering\\
      Indian Institute of Technology Gandhinagar
      \AND
      \name Anirban Dasgupta \email anirbandg@iitgn.ac.in \\
      \addr Department of Computer Science and Engineering\\
      Indian Institute of Technology Gandhinagar}

\def\month{MM}  
\def\year{YYYY} 
\def\openreview{\url{https://openreview.net/forum?id=XXXX}} 

\begin{document}

\maketitle

\begin{abstract}
Dynamic graph neural networks (DGNNs) that operate on snapshot sequences typically fall into one of two categories. \emph{Temporal-first} approaches build per-node temporal embeddings and only afterward perform spatial aggregation, whereas \emph{Spatial-first} approaches invert this order, feeding the output of a graph convolution into a downstream temporal module. In either case, the rigid sequencing forces the second stage to consume an already-compressed summary produced by the first, ruling out joint reasoning over topology and evolution; effectively, the message-passing operator never gets to weight a neighbor's contribution by that neighbor's \emph{past} trajectory. This paper introduces \textbf{SiST-GNN} (\textbf{Si}multaneous \textbf{S}patial-\textbf{T}emporal \textbf{GNN}), which fuses the two signals inside a single message-passing operation rather than chaining them. At each snapshot, we maintain a recurrent hidden state per node that summarises its history, pairs it with the node's current feature vector, and treats the pair as two nodes joined by a cross-time edge; running a standard graph convolution on this temporally augmented graph yields the updated representation. We compare against fourteen link-prediction baselines under both the fixed-split and live-update evaluation regimes, and eleven baselines on node classification. Across the public benchmarks, SiST-GNN improves on the strongest prior method in link prediction by $1-18\%$ in the fixed-split setting, and is the leading learned method on five of six datasets in the live-update setting, improving on the strongest prior method by $1-158\%$ there. We additionally derive three dynamic node-classification tasks by discretizing the underlying continuous-time event streams; here SiST-GNN beats the leading discrete-time (DTDG) baseline by $7-23\%$ and matches continuous-time (CTDG) methods that consume the raw events directly.
\end{abstract}

\section{Introduction}\label{sec:intro}
Most consequential relation systems, such as financial transaction networks, peer-to-peer trust graphs, and communication in social platforms, are inherently \emph{dynamic}: their topology and node attributes co-evolve over time \citep{JMLR:v21:19-447, 10.1145/3534678.3539300}. Forecasting the next interaction in such a system is a canonical \emph{Dynamic Link Prediction} problem, and it has been tackled by a variety of DGNNs \citep{egcn, Xu2020Inductive, tgn_icml_grl2020, sankar2020dysat, trivedi2018dyrep, yu2023towards}. In an analogous setting, a subset of nodes may exhibit anomalous behavior relative to the rest of the network. Consequently, detecting these anomalies is formulated as a \emph{dynamic node classification} task over the snapshot sequence. DGNNs \citep{tgn_icml_grl2020, trivedi2018dyrep, 10.1145/3448016.3457564, 10.1007/978-3-030-04167-0_33, 10.1145/3292500.3330895} are frequently employed to address this challenge.

When the input is discretized into a sequence of snapshots $\{\mathcal{G}_1, \dots,\mathcal{G}_T\}$ as is standard for trust networks, citation streams, and routing data, existing DGNNs almost universally adopt one of two architectural paradigms.

\begin{itemize}

    \item
        \textbf{Temporal-First (T$\to$S):} A recurrent or attention-based module first encodes the node feature \emph{trajectory}, and the resulting temporal summary is then passed to a GNN. Representative examples include EvolveGCN \citep{egcn}, where an RNN evolves the GNN's \emph{weights}.
    \item
        \textbf{Spatial-first (S$\to$T):} A GNN is applied per snapshot, and the resulting structural embeddings are further processed by a temporal module such as a GRU or LSTM. This is the dominant pattern, instantiated by GCRN \citep{10.1007/978-3-030-04167-0_33}, T-GCN \citep{8809901}, ROLAND \citep{10.1145/3534678.3539300} and TMetaNet \citep{li2025tmetanet}. 
\end{itemize}

What both paradigms share is a strict \emph{ordering}. One modality is consumed in full before the other is allowed to look at it. The downstream module can therefore only influence the final representation via the transformed output of the upstream module, and, crucially, it cannot modulate which information the upstream module retains. A spatial-first model, for example, cannot let its message-passing operator condition the aggregation of a node $v$'s feature on $v$'s historical activity, because no temporal information has been computed yet. Simultaneously, a temporal-first model cannot let its recurrent cell condition on the current structural neighborhood. We argue that this rigidity is responsible for a substantial part of the gap between current DGNNs. 

We propose that the two signals should interact within a single message-passing step. For every node $v$, we maintain a recurrent hidden state $\mathbf{h}_v^{(t)}$ that summarizes its trajectory, and at every snapshot we build a \emph{temporally augmented graph} with $N$ additional vertices, where the first $N$ carry the current spatial features $\mathrm{proj}(\mathbf{X}_t)$, and the next $N$ carry the recurrent summaries $\widetilde{\mathbf{X}}_t$. For every original edge $(u,v)\in\mathcal{E}_t$ we add a cross-time edge $(u+N, v)$ and $(u+N, u)$, so that each node receives both the present-time feature and the temporal summary of every neighbor and itself, all within one graph convolution. The output of the first $N$ vertices is the next-layer representation.

This simple construction has two consequences: first, the GNN operates on a richer message pool where each node sees $|\mathcal{N}(u)|$ structural messages \emph{and} $|\mathcal{N}(u)|+1$ temporal messages per layer, with the feature sets distinguishable by the message-passing operator through a standard aggregation function. Second, the assignment between spatial and temporal information becomes \emph{data-dependent and per-neighbor}: a node can attend more to a neighbor's recent history when its current features are uninformative, and conversely fall back on present structure when temporal context is stale.

\paragraph{Contributions.}
We make the following contributions:
\begin{itemize}
    \item We identify an architectural bottleneck shared by all snapshot based DGNNs that imposes ordering between spatial and temporal computation, and motivate \emph{simultaneous} spatial-temporal message passing as a third paradigm (Section~\ref{sec:method}).
    \item We instantiate this paradigm as \textbf{SiST-GNN}: a stackable layer that fuses a recurrent cell with a graph convolution over a temporally augmented graph (Section~\ref{sec:method}).
    \item We characterise the resulting function class. We show that a SiST-GNN stack recovers both sequential paradigms as special cases of its edge-type gates, that the inclusion is \emph{strict} via an explicit witness function that neither paradigm can represent, and that the layer is permutation-equivariant despite the index-dependent node augmentation (Section~\ref{sec:theory}).
    \item We evaluate SiST-GNN against fourteen baselines, spanning static GNNs, snapshot DGNNs, evolving-weight architectures, ROLAND variants, and the meta-learner TMetaNet \citep{li2025tmetanet}, on six public benchmarks under fixed-split \emph{and} live-update protocols. SiST-GNN attains the best MRR of any learned method on every fixed-split dataset and on five of the six live-update datasets, improving over the strongest baseline by $1.3-18.2\%$ (fixed-split) and by up to $158\%$ (live-update).
    \item We further demonstrate the effectiveness of the our architecture for \emph{dynamic node classification} by discretizing continuous-time event streams into fixed-width snapshots. SiST-GNN improves test AUC by $7-23\%$ over the discrete-time baseline and achieves comparable performance to continuous-time models operating on the native event stream (Section \ref{sec:nc_results}).

\end{itemize}

\begin{figure*}[t]
\centering
\tikzset{
  nbox/.style   = {circle, draw=black!70, line width=0.75pt, fill=blue!10,
                   minimum size=0.65cm, inner sep=0pt, font=\small},
  obox/.style   = {circle, draw=black!70, line width=0.6pt, fill=orange!22,
                   minimum size=0.78cm, inner sep=0pt, font=\small},
  hbox/.style   = {circle, draw=black!70, line width=0.7pt, fill=red!10,
                   minimum size=0.5cm, inner sep=0pt, font=\scriptsize},
  tlayer/.style = {draw=orange!55!black, line width=0.5pt, rounded corners=2.2pt,
                   fill=yellow!28, minimum width=1.05cm, minimum height=0.5cm,
                   font=\scriptsize\bfseries, align=center, inner sep=1.5pt},
  arrS/.style   = {-{Stealth[length=1.6mm]}, color=blue!55!black, line width=0.7pt},
  arrT/.style   = {-{Stealth[length=1.6mm]}, color=orange!75!red, line width=0.7pt},
  arrO/.style   = {-{Stealth[length=2.0mm]}, color=black, line width=0.9pt},
  panellbl/.style = {font=\small\bfseries, align=center},
  legendlbl/.style = {font=\scriptsize, align=left}
}
\begin{minipage}[t]{0.32\textwidth}\centering
\begin{tikzpicture}
  \node[nbox] (Xa) at (-1.4, 1.6) {$x_a$};
  \node[nbox] (Xb) at ( 1.4, 1.6) {$x_b$};
  \node[nbox] (Xc) at ( 0.0,-1.75) {$x_c$};
  \node[tlayer] (TLa) at (-0.95, 0.78) {T-Layer};
  \node[tlayer] (TLb) at ( 0.95, 0.78) {T-Layer};
  \node[tlayer] (TLc) at ( 0.00,-0.85) {T-Layer};
  \node[nbox] (Xi) at (0,0) {$x_i$};
  \node[obox] (Xtil) at (2.05,0) {$\widetilde{x}_i$};
  \draw[arrT] (Xa.south east)  -- (TLa.north);
  \draw[arrS] (TLa.south east) -- (Xi.north west);
  \draw[arrT] (Xb.south west)  -- (TLb.north);
  \draw[arrS] (TLb.south west) -- (Xi.north east);
  \draw[arrT] (Xc.north)  -- (TLc.south);
  \draw[arrS] (TLc.north) -- (Xi.south);
  \draw[arrO] (Xi.east)   -- (Xtil.west);
\end{tikzpicture}\\[1pt]
{\small\textbf{(a) Temporal $\to$ Spatial}}\\
{\scriptsize\itshape per-neighbor}
\end{minipage}\hfill
\begin{minipage}[t]{0.32\textwidth}\centering
\begin{tikzpicture}
  \node[nbox] (Xa) at (-1.3, 1.45) {$x_a$};
  \node[nbox] (Xb) at ( 1.3, 1.45) {$x_b$};
  \node[nbox] (Xc) at ( 0.0,-1.55) {$x_c$};
  \node[nbox] (Xi) at (0,0) {$x_i$};
  \node[tlayer] (TL) at (1.25,0) {T-Layer};
  \node[obox]   (Xtil) at (2.55,0) {$\widetilde{x}_i$};
  \draw[arrS] (Xa.south east) -- (Xi.north west);
  \draw[arrS] (Xb.south west) -- (Xi.north east);
  \draw[arrS] (Xc.north) -- (Xi.south);
  \draw[arrT] (Xi.east)  -- (TL.west);
  \draw[arrO] (TL.east)  -- (Xtil.west);
\end{tikzpicture}\\[1pt]
{\small\textbf{(b) Spatial $\to$ Temporal}}\\
{\scriptsize\itshape per-target}
\end{minipage}\hfill
\begin{minipage}[t]{0.34\textwidth}\centering
\begin{tikzpicture}
  \node[nbox] (Xa) at (-1.55, 1.55) {$x_a$};
  \node[nbox] (Xb) at ( 1.35, 1.55) {$x_b$};
  \node[nbox] (Xc) at ( -0.55,-1.70) {$x_c$};
  \node[hbox] (Ha) at (-1.55, 0) {$\tilde{x}_a$};
  \node[hbox] (Hb) at ( 0, 1) {$\tilde{x}_b$};
  \node[hbox] (Hc) at (0.6,-0.85) {$\tilde{x}_c$};
  \node[tlayer] (TLa) at (-2, 0.75) {T-Layer};
  \node[tlayer] (TLb) at ( 0, 2) {T-Layer};
  \node[tlayer] (TLc) at ( 0.75,-1.70) {T-Layer};
  \node[nbox]  (Xi)   at (0,0) {$x_i$};
  \node[obox]  (Xtil) at (2.30,0) {$\widetilde{x}_i$};
  \draw[arrS] (Xa.south east) -- (Xi.north west);
  \draw[arrS] (Xb.south west) -- (Xi.north east);
  \draw[arrS] (Xc.north) -- (Xi.south west);
  \draw[arrT] (Xa.south west) --  (TLa.north);
  \draw[arrT] (TLa.south) -- (Ha.north west);
  \draw[arrT] (Xb.north west)  -- (TLb.east);
  \draw[arrT] (TLb.south) -- (Hb.north);
  \draw[arrT] (Xc.east) -- (TLc.west);
  \draw[arrT] (TLc.north) -- (Hc.south);
  \draw[arrT, densely dashed] (Ha.east) -- (Xi.west);
  \draw[arrT, densely dashed] (Hb.south) -- (Xi.north);
  \draw[arrT, densely dashed] (Hc.north west)      -- (Xi.south east);
  \draw[arrO] (Xi.east) -- (Xtil.west);
\end{tikzpicture}\\[1pt]
{\small\textbf{(c) Simultaneous (Ours)}}\\
{\scriptsize\itshape per-neighbor, per-modality}
\end{minipage}
\\[3pt]
\begin{tikzpicture}[baseline]
  \node[font=\scriptsize] at (0,0)
    {\textcolor{blue!55!black}{\rule[1pt]{0.5cm}{0.7pt}}~spatial msg.\quad
     \textcolor{orange!75!red}{\rule[1pt]{0.5cm}{0.7pt}}~temporal flow\quad
     \textcolor{orange!75!red}{\rule[1pt]{0.18cm}{0.7pt}~\rule[1pt]{0.18cm}{0.7pt}}~cross-time msg.\ ($\hat{\mathcal{E}}_t$)\quad
     \fbox{\scriptsize T-Layer} = per-node recurrent (LSTM) cell};
\end{tikzpicture}
\caption{The three paradigms, viewed at the target node $x_i$ with neighbors $\{x_a, x_b, x_c\}$. (a)~\emph{Temporal-first}. (b)~\emph{Spatial-first}. (c)~Our \emph{simultaneous} paradigm, in which every node carries a temporal counterpart $\tilde{x}_v$ and both views reach $x_i$ within one graph convolution over $\hat{\mathcal{G}}_t$: solid intra-time edges, dashed cross-time edges.}
\label{fig:paradigms}
\end{figure*}

\section{Related Work}\label{sec:related}

\paragraph{Continuous-time DGNNs.}
A first family of DGNN models the graph as a stream of timestamped events. TGN~\citep{tgn_icml_grl2020}, TGAT~\citep{Xu2020Inductive}, JODIE~\citep{10.1145/3292500.3330895}, DyRep~\citep{trivedi2018dyrep}, CAW~\citep{wang2021inductive}, and APAN~\citep{10.1145/3448016.3457564} update node memories at every event using time-aware attention or recurrence, and report strong empirical performance on fine-grained interaction data. DyGFormer~\citep{yu2023towards} instead learns from each node's historical first-hop interaction sequence with a Transformer, augmented by neighbor co-occurrence features. Continuous-time approaches require per-edge temporal encodings and are less natural when the data is intrinsically snapshot-structured~\citep{10.5555/3600270.3602656,huang2023temporal,cong2023do}. A recent survey~\citep{11202740} empirically compares representative dynamic GNN models and training frameworks on standard dynamic graph datasets, reporting accuracy alongside training efficiency and GPU memory usage.

\paragraph{Snapshot-based DGNNs: temporal-first.}
In EvolveGCN-H and EvolveGCN-O~\citep{egcn}, an RNN evolves the GCN \emph{weights} across snapshots, while the per-snapshot graph convolution remains static.

\paragraph{Snapshot-based DGNNs: spatial-first.}
GCRN~\citep{10.1007/978-3-030-04167-0_33} and T-GCN~\citep{8809901} first apply spectral or spatial convolutions and then aggregate the embeddings with a GRU or LSTM. DySAT~\citep{sankar2020dysat} replaces the recurrence with self-attention over snapshot embeddings. \citet{10.1145/3534678.3539300} generalize the pattern with ROLAND, observing that any static GNN can be ``dynamized" by recurrently updating its hierarchical node states between snapshots, and introduce both a scalable incremental-training procedure and a live-update evaluation protocol that we adopt. Decoupled architectures~\citep{10.14778/3598581.3598595} further refine this template by separating feature propagation from the learned temporal model. TMetaNet~\citep{li2025tmetanet} augments the spatial-first pipeline with Dowker zigzag persistence diagrams that guide a meta-learning parameter-update rule. Like all spatial-first models, TMetaNet computes a per-snapshot structural embedding before any temporal signal has interacted with the message passing.

\paragraph{Other related ideas.}
Dynamic embedding methods~\citep{10.5555/3454287.3455247,GOYAL2020104816,goyal2018dyngem,10.1145/3184558.3191526} and meta-learning-based dynamic models~\citep{10.1145/3580305.3599551,li2025tmetanet} share with us the high-level goal of better leveraging temporal context, but all operate within the sequential paradigm. Among the latter, WinGNN~\citep{10.1145/3580305.3599551} drops explicit temporal encodings altogether and instead aggregates stochastic gradients across a sliding window of snapshots.\footnote{We were unable to reproduce the WinGNN MRR results. We therefore omit WinGNN from our results tables and treat it as related work.} Our construction of an augmented edge set bears a superficial resemblance to heterogeneous graph design~\citep{wang2019han}, where multiple semantic edge types are introduced based on domain knowledge; however, in our case, cross-time edges are added deterministically from the existing edge set without external annotation.

\paragraph{Dynamic node classification.}
A complementary line of work targets \emph{dynamic node classification} on the JODIE benchmarks~\citep{10.1145/3292500.3330895}, which are released as continuous-time event streams. Continuous-time models such as TGAT~\citep{Xu2020Inductive} and TGN~\citep{tgn_icml_grl2020} consume these streams directly, while TREND~\citep{10.1145/3485447.3512164} and GraphMixer~\citep{cong2023do} target the same setting through Hawkes-process and MLP-mixer formulations, respectively. More recently, \citet{yu2025nodetime} learn a shared encoder upstream and adapt it to downstream tasks through node- and time-conditional prompts. Our work is orthogonal to this pre-training direction: we treat the task as supervised snapshot-based classification and do not pre-train on auxiliary objectives. We discuss this gap in Section~\ref{sec:limitations}.

\section{Preliminaries and Problem Formulation}\label{sec:prelim}

\begin{definition}[Dynamic graph sequence]
A \emph{dynamic graph sequence} is an ordered collection
$\{\mathcal{G}_t = (\mathcal{V}_t, \mathcal{E}_t, \mathbf{X}_t)\}_{t=1}^{T}$,
where $\mathcal{V}_t$ is the node set at snapshot $t$ (with $|\mathcal{V}_t|=N$),
$\mathcal{E}_t \subseteq \mathcal{V}_t \times \mathcal{V}_t$ is the edge set,
and $\mathbf{X}_t \in \mathbb{R}^{N \times d}$ encodes node features at $t$.
\end{definition}

\noindent We work with the standard ROLAND~\citep{10.1145/3534678.3539300} setting, in which nodes are indexed consistently across time, where a node absent in snapshot $t$ has no incident edges in $\mathcal{E}_t$ but retains its identifier. Thus $\mathcal{V}_t = \mathcal{V}$ for all $t$. The following are two of the tasks we are interested in. 

For \emph{future link prediction} experiments: given $\{\mathcal{G}_1,\ldots,\mathcal{G}_t\}$, predict the probability of an edge $(u,v) \in \mathcal{V}\times \mathcal{V}$ being present in $\mathcal{G}_{t+1}$.

For \emph{node classification} of anomalous nodes: $\{\mathcal{G}_1,\ldots,\mathcal{G}_t\}$, predict the probability of a node being anomalous in $\mathcal{G}_{t+1}$. 

\begin{definition}[Temporal-first paradigm]\label{def:tf}
Given snapshot features $\mathbf{X}_t$ and previous state $\mathbf{H}_{t-1}$, the temporal-first paradigm computes
\begin{equation}
  \mathbf{Z}_t = \mathrm{GNN}\bigl(
      f_{\mathrm{temp}}(\mathbf{X}_t, \mathbf{H}_{t-1}),
      \mathcal{E}_t
  \bigr)
\end{equation}
where $f_{\mathrm{temp}}$ is a recurrent or attention-based temporal encoder.
\end{definition}

\begin{definition}[Spatial-first paradigm]\label{def:sf}
The spatial-first paradigm computes
\begin{equation}
  \mathbf{Z}_t = f_{\mathrm{temp}}\bigl(
      \mathrm{GNN}(\mathbf{X}_t, \mathcal{E}_t),
      \mathbf{H}_{t-1}
  \bigr)
\end{equation}
\end{definition}

\section{The SiST-GNN Architecture}\label{sec:method}

We now formalize our SiST-GNN layer, instantiate its temporal encoder, and analyze complexity.

\subsection{Temporal Encoder}\label{sec:temporal}

Each node $v$ keeps hidden and cell states $(\mathbf{h}_v^{(t)}, \mathbf{c}_v^{(t)}) \in \mathbb{R}^{d_h}\times\mathbb{R}^{d_h}$ that summarise its trajectory up to snapshot $t-1$. At snapshot $t$ the LSTM cell~\citep{10.1162/neco.1997.9.8.1735} consumes the current feature vector and updates the states,
\begin{equation}\label{eq:lstm}
  (\mathbf{h}_v^{(t)},\, \mathbf{c}_v^{(t)}) =
    \mathrm{LSTMCell}\bigl(
      \mathbf{x}_v^{(t)},
      (\mathbf{h}_v^{(t-1)}, \mathbf{c}_v^{(t-1)})
    \bigr)
\end{equation}
and we set the temporal summary to the current hidden state, $\widetilde{\mathbf{x}}_v^{(t)} \equiv \mathbf{h}_v^{(t)}$. The LSTM cell is shared across all $N$ nodes (its parameters do not scale with $|\mathcal{V}|$) and is applied row-wise, which preserves the permutation equivariance of the layer (Lemma ~\ref{lem:equiv}). We follow standard truncated BPTT~\citep{10.1162/neco.1990.2.4.490}: hidden and cell states are detached after every backward pass to bound the temporal credit-assignment window and to keep per-snapshot memory constant in $T$.

\subsection{Temporally Augmented Graph}\label{sec:aug}

\begin{definition}[Augmented graph]\label{def:aug}
Given $\mathcal{G}_t = (\mathcal{V}_t, \mathcal{E}_t)$ with $|\mathcal{V}_t|=N$, the \emph{temporally augmented graph} is
$\hat{\mathcal{G}}_t = (\hat{\mathcal{V}}_t, \hat{\mathcal{E}}_t)$ with
\begin{align}
  &\hat{\mathcal{V}}_t = \{1,\ldots,N\} \cup \{N{+}1,\ldots,2N\}\\
  &\hat{\mathcal{E}}_t = \mathcal{E}_t \cup \{(u+N, v) \mid (u,v)\in\mathcal{E}_t\} \cup \{(u+N, u) \mid u \in \mathcal{V}_t\}\\
  &\mathbf{X}_{\mathrm{aug}} = [\mathrm{proj}(\mathbf{X}_t)\,;\, \widetilde{\mathbf{X}}_t]
\end{align}
where $\mathrm{proj}(\mathbf{X}_t)$ is the projected node features in $d_h$ dimension and $\widetilde{\mathbf{X}}_t$ is the temporal encoded features.
\end{definition}

\noindent The cross-time edges $(u+N, v)$ and $(u+N, u)$ allow node $u$ to receive directly, via one round of standard message passing, the temporal summary of its neighbor $v$ and itself. Only original edges (intra-time edges) induce cross-time edges, so the augmented graph is always a sparse perturbation of the input ($|\hat{\mathcal{E}}_t| = 2|\mathcal{E}_t| + N$).

\subsection{Forward Pass}

One SiST-GNN layer computes
\begin{equation}\label{eq:forward}
  \mathbf{Z}_t = \sigma\Bigl(
    \mathrm{GNN}\bigl(
      \mathbf{X}_{\mathrm{aug}},
      \hat{\mathcal{E}}_t
    \bigr)_{[1:N]}
  \Bigr)
\end{equation}
where the $\sigma$ operator is element-wise nonlinearity. The first $N$ rows of the GNN output are the next-layer input; the lower $N$ rows are discarded. Algorithm~\ref{alg:layer} gives the full single-snapshot procedure. When $L$ layers are stacked, each layer keeps its own recurrent state, and the cross-time edges are re-instantiated at every depth.
\begin{algorithm}[t]
\caption{SiST-GNN layer, single snapshot.}
\label{alg:layer}
\begin{algorithmic}[1]
\Require Features $\mathbf{X}_t \in \mathbb{R}^{N \times d}$;
         the COO encoding $\mathbf{E}_t \in \mathbb{Z}^{2\times|\mathcal{E}_t|}$ of the
         edge set $\mathcal{E}_t$ (column $k$ holds the endpoints of the $k$-th edge);
         prior states $\mathbf{H}_{t-1}, \mathbf{C}_{t-1} \in \mathbb{R}^{N \times d_h}$.
\Ensure  Output $\mathbf{Z}_t \in \mathbb{R}^{N \times d_h}$, updated states $\mathbf{H}_t, \mathbf{C}_t$.
\State $\mathbf{H}_t, \mathbf{C}_t \gets \mathrm{LSTMCell}(\mathbf{X}_t, (\mathbf{H}_{t-1}, \mathbf{C}_{t-1}))$ \Comment{Temporal update}
\State $\mathbf{X}_{\mathrm{proj}} \gets W_p \mathbf{X}_t$ \Comment{Project to $d_h$}
\State $\mathbf{X}_{\mathrm{aug}} \gets [\mathbf{X}_{\mathrm{proj}}; \mathbf{H}_t]$ \Comment{$2N \times d_h$}
\State $\hat{\mathbf{E}}_t \gets
        \bigl[\,\mathbf{E}_t\big\|
        (\mathbf{E}_t^{[0,:]}+ N,\mathbf{E}_t^{[1,:]} )\,\big\|\{(N{+}i,\,i)\}_{i=1}^{N}]$
        \Comment{Append cross-time columns ($N$-shifted source)}
\State $\mathbf{X}' \gets \mathrm{GNN}(\mathbf{X}_{\mathrm{aug}}, \hat{\mathbf{E}}_t)$ \Comment{Joint message passing}
\State $\mathbf{Z}_t \gets \sigma(\mathbf{X}'_{[1:N]})$ \Comment{Slice original nodes}
\State \Return $\mathbf{Z}_t, \mathbf{H}_t, \mathbf{C}_t$
\end{algorithmic}
\end{algorithm}
\subsection{Link Prediction Head}\label{sec:lp_head}

Following the ROLAND convention, the per-snapshot input is a set of per-node learnable embeddings $\mathbf{P}\in\mathbb{R}^{N\times d_h}$ trained jointly with the model, so that $\mathbf{X}_t=\mathbf{P}$ at every $t$. The encoder is an $L$-layer stack of SiST-GNN layers with ReLU and dropout in between, each layer carrying its own $(\mathbf{H}^{(\ell)}_t,\mathbf{C}^{(\ell)}_t)$; its output $\mathbf{Z}_t$ is the node embedding used downstream.

Candidate edges are scored by a parameter-free inner product,
\begin{equation}\label{eq:decoder}
  s_t(u, v) = \langle \mathbf{z}_u^{(t)},\, \mathbf{z}_v^{(t)} \rangle,
\end{equation}
and we optimise the margin-ranking loss with unit margin over the positive edges of the current snapshot,
\begin{equation}\label{eq:margin}
  \mathcal{L}_t = \frac{1}{|\mathcal{E}_t|}\sum_{(u,v)\in\mathcal{E}_t}
    \max\Bigl(0, 1 - s_t(u,v) + s_t(u,v^-)\Bigr).
\end{equation}
We sample one negative tail $v^- \sim \mathrm{Uniform}(\mathcal{V}_t)$ per positive edge during training and $1000$ during evaluation, and report the mean reciprocal rank (MRR) of the positive against those negatives. Every $(\mathbf{H}^{(\ell)}_t,\mathbf{C}^{(\ell)}_t)$ is detached before snapshot $t{+}1$, the same truncation used by ROLAND's incremental scheme.

\subsection{Node Classification Head}\label{sec:nc_head}

The same encoder is reused for node classification; only the input pipeline, readout, and loss change.

\paragraph{Discretizing continuous-time interactions.}
The JODIE-style benchmarks (Wikipedia, Reddit, MOOC) are released as continuous-time bipartite interaction streams $\{(u_k, v_k, t_k, \mathbf{e}_k, y_k)\}_{k=1}^{|\mathcal{E}|}$, where $\mathbf{e}_k\in\mathbb{R}^{d_e}$ is the edge feature and $y_k\in\{0,1\}$ is the state-change label of the source node $u_k$ at time $t_k$. We bin all interactions into fixed-width windows of $\Delta$ hours, indexed $t=1,\ldots,T$, so that $\mathcal{G}_t$ is the multigraph of interactions with timestamps in $[(t{-}1)\Delta,\, t\Delta)$; a node carries its constant raw feature in every snapshot. We sweep $\Delta\in\{1, 3, 6, 12, 24\}$~h (Figure~\ref{fig:ablations_nc}c) and default to $\Delta=6$~h. Discretization trades the fine-grained event ordering used by continuous-time models~\citep{tgn_icml_grl2020, Xu2020Inductive} for a uniform snapshot interface; we return to this in Section~\ref{sec:limitations}.

\paragraph{Input features, readout, and loss.}
Each node receives its dataset-provided static feature $\mathbf{f}_v\in\mathbb{R}^{d_f}$ plus a learnable embedding $\mathbf{p}_v\in\mathbb{R}^{d_h}$, projected and summed into a single $\mathbb{R}^{d_h}$ vector; edge features are encoded by a two-layer MLP and fused into the source representation by a residual connection before the readout. A two-layer MLP $\phi:\mathbb{R}^{d_h}\to\mathbb{R}$ scores each source node, $\hat{y}_v^{(t)} = \sigma(\phi(\mathbf{z}_v^{(t)}))$, and we optimise the weighted binary cross-entropy
\begin{equation}\label{eq:bce}
  \mathcal{L}_t^{\mathrm{NC}} = -\frac{1}{|\mathcal{S}_t|}\sum_{v\in\mathcal{S}_t}
  \Bigl(w^{+}\, y_v^{(t)}\log \hat{y}_v^{(t)} + (1-y_v^{(t)})\log(1-\hat{y}_v^{(t)})\Bigr)
\end{equation}
over the set $\mathcal{S}_t$ of source nodes emitting at least one interaction in snapshot $t$. Following standard practice on these benchmarks we set $w^{+}=|\mathcal{S}_t^{-}|/|\mathcal{S}_t^{+}|$ (\emph{balanced} BCE) to compensate for the extreme positive-class imbalance, and use the same truncated-BPTT schedule as above.

\subsection{Theoretical Properties}\label{sec:theory}

We prove three properties of the SiST-GNN layer: two reduction lemmas showing that a SiST-GNN stack can simulate either of the sequential paradigms of Section~\ref{sec:prelim} (Lemmas~\ref{lem:sf} and~\ref{lem:tf}); a strict-generalization theorem that follows from the lemmas together with an explicit witness (Theorem~\ref{thm:subsume}); and permutation equivariance (Lemma~\ref{lem:equiv}).

\paragraph{Edge-type gates.}
The reduction lemmas rest on a standard assumption from heterogeneous-edge GNNs: the spatial backbone weights the edge types of $\hat{\mathcal{E}}_t$ separately. Node $u\in\{1,\dots,N\}$ receives four kinds of incoming message in $\hat{\mathcal{G}}_t$, and we attach one scalar gate to each,
\begin{itemize}
  \item $\alpha_{\mathrm{intra}}$, on the intra-time edges $(v,u)\in\mathcal{E}_t$, carrying the neighbours' current features;
  \item $\alpha_{\mathrm{cross}}$, on the cross-time edges $(v{+}N,u)$, carrying the neighbours' temporal summaries;
  \item $\alpha_{\mathrm{self}}^{\mathrm{cur}}$, on the self-loop $(u,u)$, carrying $u$'s own current feature;
  \item $\alpha_{\mathrm{self}}^{\mathrm{cross}}$, on the cross-time self edge $(u{+}N,u)$, carrying $u$'s own temporal summary.
\end{itemize}
Each gate multiplies the contribution of its edge type to the aggregation sum, all four equal $1$ by default, and setting a gate to $0$ drops the corresponding messages. The gates add no asymptotic complexity and are subsumed by per-edge-type weight matrices in any standard message-passing backbone (GCN, GraphSAGE, GAT, GIN). We write $\alpha^{(\ell)}$ for the gates of layer $\ell$.

\paragraph{Recurrence assumption.}
We further assume the temporal cell $f_{\mathrm{temp}}$ admits an \emph{identity parameterization}: a setting of its parameters under which $f_{\mathrm{temp}}(\mathbf{x},\mathbf{h})=\mathbf{h}$ for every $(\mathbf{x},\mathbf{h})$. This is satisfied by every gated recurrent variant used in the DGNN literature; for the LSTM, the forget gate is driven to $1$, the input gate to $0$, and the output gate aligned with the cell-state activation.

\noindent Before the reductions, note the message pool the gates act on. Node $u$ receives $2|\mathcal{N}(u)|+1$ incoming messages per layer, against $|\mathcal{N}(u)|+1$ for a sequential DGNN with the same backbone. The two messages originating from a neighbour $v$ carry $(\mathbf{X}_t\mathbf{W}_p)[v,:]$ and $\widetilde{\mathbf{x}}_v^{(t)}=\mathbf{h}_v^{(t)}$; the first is a linear function of $\mathbf{x}_v^{(t)}$ alone while the second depends on the whole history through the recurrence, so no function of $\mathbf{x}_v^{(t)}$ can equal both and the aggregator is free to weight them independently. This is the mechanism by which the operator can, at inference time, weight a neighbour's present features and its historical trajectory differently for a given prediction, without any architectural change.

\begin{lemma}[Spatial-first recovery]\label{lem:sf}
There exist parameter settings of two stacked SiST-GNN layers under which, for every $t\ge 1$ and every initial state $\mathbf{H}_0$, the layer-$2$ output equals the spatial-first function $\mathbf{Z}_t=f_{\mathrm{temp}}(\mathrm{GNN}(\mathbf{X}_t,\mathcal{E}_t),\mathbf{H}_{t-1})$ of Definition~\ref{def:sf}.
\end{lemma}

\begin{proof}
Let $\mathbf{H}^{(\ell)}_t$ denote the hidden state of the layer-$\ell$ recurrent cell, $\ell\in\{1,2\}$, at snapshot $t$. We construct parameters such that
\[
  \mathbf{Z}^{(1)}_t = \mathrm{GNN}(\mathbf{X}_t,\mathcal{E}_t),\qquad
  \mathbf{Z}^{(2)}_t = f_{\mathrm{temp}}\!\bigl(\mathbf{Z}^{(1)}_t,\mathbf{H}^{(2)}_{t-1}\bigr),
\]
and verify that $\mathbf{H}^{(2)}_{t-1}$ tracks the spatial-first reference state at every $t$.

\emph{Layer $1$.} Set $\alpha^{(1)}_{\mathrm{intra}}=\alpha^{(1),\mathrm{cur}}_{\mathrm{self}}=1$ and $\alpha^{(1)}_{\mathrm{cross}}=\alpha^{(1),\mathrm{cross}}_{\mathrm{self}}=0$, so that all cross-time messages are dropped and the convolution sees only the upper half of $\mathbf{X}_{\mathrm{aug}}$, i.e.\ $\mathbf{X}_t\mathbf{W}_p^{(1)}$ over $\mathcal{E}_t$ with self-loops. Choose $\mathbf{W}_p^{(1)}$ together with the backbone's own input map to realize the reference GNN: every standard message-passing backbone has a learnable input linear map, so the composition $(\mathbf{X}_t\mathbf{W}_p^{(1)})\mapsto\mathrm{GNN}(\cdot)$ expresses the same function class as $\mathrm{GNN}(\mathbf{X}_t,\cdot)$. This gives us $\mathbf{Z}^{(1)}_t=\mathrm{GNN}(\mathbf{X}_t,\mathcal{E}_t)$. The layer-$1$ recurrent state is updated but read by no downstream computation, so its parameters are free.

\emph{Layer $2$.} Set $\alpha^{(2),\mathrm{cross}}_{\mathrm{self}}=1$ and $\alpha^{(2)}_{\mathrm{intra}}=\alpha^{(2)}_{\mathrm{cross}}=\alpha^{(2),\mathrm{cur}}_{\mathrm{self}}=0$. The only message reaching node $u$ is then its own cross-time copy, whose feature is the layer-$2$ temporal summary $f_{\mathrm{temp}}(\mathbf{Z}^{(1)}_t,\mathbf{H}^{(2)}_{t-1})[u,:]$. With $\mathbf{W}_p^{(2)}=\mathbf{I}$, an identity weight on that edge type, and $\sigma^{(2)}=\mathrm{id}$, the layer-$2$ output is $\mathbf{Z}^{(2)}_t=f_{\mathrm{temp}}(\mathbf{Z}^{(1)}_t,\mathbf{H}^{(2)}_{t-1})$.

\emph{State alignment, by induction on $t$.} Initialize $\mathbf{H}^{(2)}_0=\mathbf{H}_0$. At $t=1$, $\mathbf{Z}^{(2)}_1=f_{\mathrm{temp}}(\mathrm{GNN}(\mathbf{X}_1,\mathcal{E}_1),\mathbf{H}_0)$, which matches the spatial-first output at $t=1$ exactly. The layer-$2$ cell produces $\mathbf{H}^{(2)}_1$ by applying $f_{\mathrm{temp}}$ to the same arguments, which by definition coincides with the reference state $\mathbf{H}_1$. The induction step from $t$ to $t+1$ is identical.
\end{proof}

\noindent Together with its temporal-first counterpart below, this says the construction gives up nothing: whatever a spatial-first model computes, a SiST-GNN stack can compute too.

\begin{lemma}[Temporal-first recovery]\label{lem:tf}
There exists a parameter setting of a single SiST-GNN layer under which, for every $t\ge 1$ and every initial state $\mathbf{H}_0$, the output equals the temporal-first function $\mathbf{Z}_t=\mathrm{GNN}(f_{\mathrm{temp}}(\mathbf{X}_t,\mathbf{H}_{t-1}),\mathcal{E}_t)$ of Definition~\ref{def:tf}.
\end{lemma}

\begin{proof}
Set $\alpha_{\mathrm{cross}}=\alpha_{\mathrm{self}}^{\mathrm{cross}}=1$ and $\alpha_{\mathrm{intra}}=\alpha_{\mathrm{self}}^{\mathrm{cur}}=0$. Every message that survives then originates in the lower half of $\mathbf{X}_{\mathrm{aug}}$, namely $\widetilde{\mathbf{X}}_t=f_{\mathrm{temp}}(\mathbf{X}_t,\mathbf{H}_{t-1})$: node $u$ receives the temporal summary of each $v\in\mathcal{N}(u)$ along the cross-time edge $(v{+}N,u)$ and its own along $(u{+}N,u)$. Since the cross-time edges are in bijection with $\mathcal{E}_t$ by construction (Definition~\ref{def:aug}) and the cross-time self edges play the role of self-loops, the aggregation is exactly a convolution over $\mathcal{E}_t$ applied to $\widetilde{\mathbf{X}}_t$, i.e.\ $\mathbf{Z}_t=\mathrm{GNN}(f_{\mathrm{temp}}(\mathbf{X}_t,\mathbf{H}_{t-1}),\mathcal{E}_t)$, which is Definition~\ref{def:tf}.
\end{proof}

\begin{theorem}[Strict generalization]\label{thm:subsume}
Let $\mathcal{F}_{\mathrm{SiST}}$, $\mathcal{F}_{\mathrm{S}\to\mathrm{T}}$, and $\mathcal{F}_{\mathrm{T}\to\mathrm{S}}$ denote the function classes computed, respectively, by a SiST-GNN stack, by a spatial-first DGNN, and by a temporal-first DGNN, all sharing a fixed backbone $\mathrm{GNN}$ and a fixed recurrence $f_{\mathrm{temp}}$.
Then
\begin{equation*}
    \mathcal{F}_{\mathrm{S}\to\mathrm{T}}\cup\mathcal{F}_{\mathrm{T}\to\mathrm{S}}
  \subsetneq
  \mathcal{F}_{\mathrm{SiST}}.
\end{equation*}
\end{theorem}

\begin{proof}
\emph{Inclusion.} Lemmas~\ref{lem:sf} and~\ref{lem:tf} give, for every $g\in\mathcal{F}_{\mathrm{S}\to\mathrm{T}}\cup\mathcal{F}_{\mathrm{T}\to\mathrm{S}}$, an explicit setting of the edge-type gates of a SiST-GNN stack that computes $g$: two layers in the spatial-first case, one in the temporal-first case.

\emph{Strictness, by example.}
To show that there are functions in $\mathcal{F}_{\mathrm{SiST}}$ that cannot be represented by functions in $\mathcal{F}_{\mathrm{S}\to\mathrm{T}}$ or $\mathcal{F}_{\mathrm{T}\to\mathrm{S}}$, we construct a counterexample.

Let the GNN operator be the function $f(X) = AXW$, where $A$ is an adjacency matrix of dimension $N\times N$, and $X$ and $W$ are the feature and weight matrices. Let $g$ be the temporal operator. The GNN module of $\mathcal{F}_{\mathrm{SiST}}$ requires an adjacency matrix $A'$ of dimension $2N\times 2N$, which decomposes into four blocks as $
    A' = \begin{pmatrix}
A_{11} & A_{12} \\
A_{21} & A_{22}
\end{pmatrix}$,
and, correspondingly, $\mathbf{X}_{\mathrm{aug}} = \begin{pmatrix}
X \\
g(X)
\end{pmatrix}$. The top $N\times d_h$ block of SiST-GNN is then
\begin{equation}\label{eq:theorem_eq_sist}
    \mathrm{SiST\text{-}GNN}(X) = Y(X) = A_{11}XW + A_{12}\,g(X)\,W .
\end{equation}
Let $N=2$, $d_h=1$ and $W=[1]$. The input is a column matrix $X = \begin{pmatrix}
x_1 \\
x_2
\end{pmatrix}$. For simplicity, assume the temporal operator requires no hidden state, and take the row-wise temporal operator to be a square operation, $g(X) = \begin{pmatrix}
x^2_1 \\
x^2_2
\end{pmatrix}$.
Consider $A_{11} = \begin{pmatrix}
0 & 1 \\
1 & 0
\end{pmatrix}$ and $A_{12} = \begin{pmatrix}
1 & 0 \\
0 & 1
\end{pmatrix}$. Substituting these into Equation~\ref{eq:theorem_eq_sist} gives
\begin{equation*}
    Y(X) = \begin{pmatrix}
0 & 1 \\
1 & 0
\end{pmatrix}\begin{pmatrix}
x_1 \\
x_2
\end{pmatrix} + \begin{pmatrix}
1 & 0 \\
0 & 1
\end{pmatrix}\begin{pmatrix}
x^2_1 \\
x^2_2
\end{pmatrix} = \begin{pmatrix}
x_2 + x^2_1 \\
x_1 + x^2_2
\end{pmatrix}.
\end{equation*}
\emph{Case $\mathrm{T}\rightarrow\mathrm{S}$.} Any function in $\mathcal{F}_{\mathrm{T}\to\mathrm{S}}$ has the form $f(g_{TS}(X)) = A_{11}\,g_{TS}(X)\,W_{TS}$, where the two operators act on the same graph, so the adjacency is the intra-time block $A_{11}$ used above, and $W_{TS} = [a]$ is a $1\times 1$ learnable weight because $d_h=1$. Since $g_{TS}$ is a row-wise operation, it maps $X$ to $\begin{pmatrix}
\phi_1(x_1) \\
\phi_2(x_2)
\end{pmatrix}$ for some scalar functions $\phi_1,\phi_2$. The general form of a $\mathcal{F}_{\mathrm{T}\to\mathrm{S}}$ function is therefore
\begin{equation*}
    Y_{TS}(X) = \begin{pmatrix}
0 & 1 \\
1 & 0
\end{pmatrix}\begin{pmatrix}
\phi_1(x_1) \\
\phi_2(x_2)
\end{pmatrix} a = \begin{pmatrix}
a\,\phi_2(x_2)\\
a\,\phi_1(x_1)
\end{pmatrix}.
\end{equation*}
Matching this against the target $Y(X)$ requires
\begin{equation*}
    a\,\phi_2(x_2) = x_2 + x^2_1 , \quad a\,\phi_1(x_1) = x_1 + x^2_2 .
\end{equation*}
The left-hand side of the first equation is a function of $x_2$ alone, whereas the right-hand side varies with $x_1$; the second equation fails symmetrically. No choice of $\phi_1,\phi_2$ and constant $a$ satisfies them, so there is a function in $\mathcal{F}_{\mathrm{SiST}}$ that is not in $\mathcal{F}_{\mathrm{T}\to\mathrm{S}}$.

\noindent\emph{Case $\mathrm{S}\rightarrow\mathrm{T}$.} Any function in $\mathcal{F}_{\mathrm{S}\to\mathrm{T}}$ has the form $g_{ST}(f(X)) = g_{ST}(A_{11}XW_{ST})$, with $W_{ST} = [a]$ again a $1\times 1$ learnable weight and $g_{ST}$ row-wise, so
\begin{equation*}
    Y_{ST}(X) = \begin{pmatrix}
\psi_1(a x_2)\\
\psi_2(a x_1)
\end{pmatrix}.
\end{equation*}
Matching against the target gives
\begin{equation*}
        \psi_1(a x_2) = x^2_1 + x_2, \quad \psi_2(a x_1) = x_1 + x^2_2 .
\end{equation*}
Again the left-hand side of the first equation depends on $x_2$ alone while the right-hand side varies with $x_1$, and the second fails symmetrically. Thus our target function cannot be represented by any function in $\mathcal{F}_{\mathrm{S}\to\mathrm{T}}$.

Combining the two cases completes the strict inclusion.
\end{proof}

\noindent The containment is \emph{strict}: the witness above exhibits a function that a SiST-GNN layer realises and that no spatial-first or temporal-first model with the same backbone can, so the gain is one of expressive power and not merely of parameterisation. This is the formal counterpart of the informal bottleneck argument in Section~\ref{sec:intro}.

\begin{lemma}[Permutation equivariance]\label{lem:equiv}
For every permutation $\pi$ of $\{1,\dots,N\}$ with matrix $\mathbf{P}_\pi$, the SiST-GNN layer satisfies
\begin{equation*}
    \begin{split}
        \mathrm{SiST\text{-}GNN}\bigl(\mathbf{P}_\pi\mathbf{X}_t,\,\pi(\mathcal{E}_t),\,\mathbf{P}_\pi\mathbf{H}_{t-1},\,\mathbf{P}_\pi\mathbf{C}_{t-1}\bigr) \\
  \;=\;\mathbf{P}_\pi\cdot\mathrm{SiST\text{-}GNN}\!\bigl(\mathbf{X}_t,\mathcal{E}_t,\mathbf{H}_{t-1},\mathbf{C}_{t-1}\bigr)
    \end{split}
\end{equation*}
and the same identity holds simultaneously for the updated states $\mathbf{H}_t$ and $\mathbf{C}_t$.
\end{lemma}

\begin{proof}
Define $\hat{\pi}:\{1,\dots,2N\}\to\{1,\dots,2N\}$ by $\hat{\pi}(u)=\pi(u)$ for $u\le N$ and $\hat{\pi}(u+N)=\pi(u)+N$. Then $\hat\pi$ is a permutation of the augmented node set with matrix $\mathbf{P}_{\hat\pi}=\mathrm{diag}(\mathbf{P}_\pi,\mathbf{P}_\pi)$.

\emph{(i) Temporal update.} The LSTM cell is applied row-wise to $\mathbf{X}_t$; any row-wise map commutes with any row permutation, so the layer-$1$ LSTM state becomes
$(\mathbf{P}_\pi\mathbf{H}_t,\mathbf{P}_\pi\mathbf{C}_t)$ under permuted input.

\emph{(ii) Stacking.}
$\mathbf{X}_{\mathrm{aug}}=[\mathbf{X}_t\mathbf{W}_p;\,\mathbf{H}_t]\in\mathbb{R}^{2N\times d_h}$. Under permuted input the upper block becomes $\mathbf{P}_\pi\mathbf{X}_t\mathbf{W}_p$ and the lower block $\mathbf{P}_\pi\mathbf{H}_t$, which is exactly $\mathbf{P}_{\hat\pi}\mathbf{X}_{\mathrm{aug}}$.

\emph{(iii) Edge augmentation.} The augmentation map of Definition~\ref{def:aug} depends only on the unordered edge set:
\begin{align*}
  \widehat{\pi(\mathcal{E}_t)}
  =\pi(\mathcal{E}_t)
    \cup\{(\pi(u)+N,\pi(v)):(u,v)\in\mathcal{E}_t\}
    \cup\{(\pi(u)+N,\pi(u)):u\in\mathcal{V}_t\}
  =\hat\pi(\hat{\mathcal{E}}_t),
\end{align*}
so re-running the augmentation on $\pi(\mathcal{E}_t)$ produces exactly the $\hat\pi$-image of $\hat{\mathcal{E}}_t$.

\emph{(iv) Message passing.} Every standard message-passing GNN is permutation-equivariant on its node
set~\citep{kipf2017semisupervised,velickovic2018graph,10.5555/3294771.3294869,xu2018how}; applying this with the permutation $\hat\pi$ on the augmented domain,
\[
  \mathrm{GNN}\!\bigl(\mathbf{P}_{\hat\pi}\mathbf{X}_{\mathrm{aug}},\hat\pi(\hat{\mathcal{E}}_t)\bigr)
  =\mathbf{P}_{\hat\pi}\,\mathrm{GNN}\!\bigl(\mathbf{X}_{\mathrm{aug}},\hat{\mathcal{E}}_t\bigr).
\]

\emph{(v) Slicing.} Because $\hat\pi$ maps $\{1,\dots,N\}$ to itself and restricts there to $\pi$, the row-permutation $\mathbf{P}_{\hat\pi}$ acts as $\mathbf{P}_\pi$ on the first $N$ rows. The element-wise
$\sigma$ commutes with any row permutation. Hence $\mathbf{Z}_t=\sigma(\mathbf{X}'_{1:N,\,:})$ transforms as
$\mathbf{P}_\pi\mathbf{Z}_t$.

Combining (i)--(v) gives the stated equivariance for $\mathbf{Z}_t$, and (i) gives it simultaneously for $\mathbf{H}_t,\mathbf{C}_t$.
\end{proof}

\noindent Equivariance matters here because the augmentation assigns every node an index-dependent twin at $u+N$. The lemma confirms that this bookkeeping introduces no dependence on the node ordering, so the layer remains a valid graph operator.
\subsection{Complexity}

Let a spatial-first baseline with the same backbone process $N$ nodes and $E$ edges per snapshot in $O(N d_h^2 + E d_h)$ time. One SiST-GNN layer requires:
(i)~one LSTMCell pass at $O(N d_h^2)$, identical to the baseline's temporal module;
(ii)~one graph convolution over $2N$ nodes and $2E+N$ edges at $O(N d_h^2 + E d_h)$, asymptotically equal to the baseline; and
(iii)~$O(N d_h)$ additional activation memory for the stacked features.
The per-snapshot time and memory cost is therefore within a constant factor of any spatial-first baseline with the same hidden dimension.

\noindent In practice, the constant factor is amortized by using fewer stacked layers (our 2-layer SiST-GNN outperforms the baselines), and the LSTMCell overhead is dominated by the GNN.

\begin{table}[t]
\centering
\caption{Per-snapshot, per-layer complexity. ``DZP'' is the cost of computing Dowker-Zigzag persistence diagrams in TMetaNet.}
\label{tab:complexity}
\setlength{\tabcolsep}{3pt}
\small
\begin{tabular}{lrrr}
\toprule
Model & Compute & Act.\ mem. & Rec.\ state \\
\midrule
GCN~\citep{kipf2017semisupervised}                & $E d_h + N d_h^2$ & $N d_h$ & ---  \\
EvolveGCN~\citep{egcn}   & $E d_h + N d_h^2$ & $N d_h$ & $d_h^2$ \\
GCRN~\citep{10.1007/978-3-030-04167-0_33}          & $E d_h + N d_h^2$ & $N d_h$ & $N d_h$ \\
T-GCN~\citep{8809901}              & $E d_h + N d_h^2$ & $N d_h$ & $N d_h$ \\
ROLAND-GRU~\citep{10.1145/3534678.3539300}        & $E d_h + N d_h^2$ & $N d_h$ & $N d_h$ \\
TMetaNet~\citep{li2025tmetanet}         & $E d_h + N d_h^2 + \mathrm{DZP}$ & $N d_h$ & $N d_h$ \\
\textbf{SiST-GNN (ours)}               & $\mathbf{E d_h + N d_h^2}$ & $2N d_h$ & $N d_h$ \\
\bottomrule
\end{tabular}
\end{table}

\section{Experiments}\label{sec:exp}
We first state the two link-prediction protocols and the node-classification split (Section~\ref{sec:training}), then describe the datasets (Section~\ref{sec:datasets}) and baselines, including a zero-parameter reference that calibrates what the MRR numbers mean (Section~\ref{sec:persistence}). Results follow for link prediction, node classification, and the sensitivity ablations.

\subsection{Evaluation Protocols}\label{sec:training}

We adopt the two evaluation protocols introduced by ROLAND~\citep{10.1145/3534678.3539300} for the link prediction task, and another evaluation protocol ~\citep{10.1145/3292500.3330895, Xu2020Inductive} for the node classification task.

\paragraph{Fixed-split.}
The first $90\%$ of snapshots are used to train the model; the remaining $10\%$ are held out for evaluation. The model processes the training sequence in temporal order, carrying recurrent states across snapshots and detaching them after every backward step.

\paragraph{Live-update.}
At each snapshot $t$, the model \emph{first} predicts links on $\mathcal{G}_t$ using states from $t{-}1$, \emph{then} trains on $\mathcal{G}_t$ for a fixed budget of $K$ inner epochs. This mirrors the online-deployment regime in which a model must commit to predictions before observing the new ground truth, and it isolates the \emph{adaptation} component of the model from its long-horizon training behavior.

In both protocols, the loss is~\eqref{eq:margin} and the metric is MRR computed snapshot-by-snapshot.

\paragraph{Node classification split.}
We use a chronological train/validation /test split of $70/15/15\%$ of snapshots and report test AUC.

\begin{table}[t]
\centering
\caption{Dataset statistics, as reported by ROLAND~\citep{10.1145/3534678.3539300}. $|\mathcal{V}|$ and $|\mathcal{E}|$ are the numbers of nodes and edges across all snapshots; $T$ is the number of snapshots.}
\label{tab:datasets}
\small
\begin{tabular}{lrrrl}
  \toprule
  Dataset        & $|\mathcal{V}|$ & $|\mathcal{E}|$ & $T$ & Domain \\
  \midrule
  Bitcoin-OTC~\citep{7837846}    & 5{,}881  & 35{,}592  & 279 & Trust network    \\
  Bitcoin-Alpha~\citep{7837846}  & 3{,}783  & 24{,}186  & 274 & Trust network    \\
  UCI-Message~\citep{https://doi.org/10.1002/asi.21015}    & 1{,}899  & 59{,}835  & 29  & Communication    \\
  Reddit-Title~\citep{10.1145/3178876.3186141}    & 54{,}075 & 571{,}927 & 178  & Hyperlink graph  \\
  Reddit-Body~\citep{10.1145/3178876.3186141}     & 35{,}776 & 286{,}561 & 178  & Hyperlink graph  \\
  AS-733~\citep{10.1145/1081870.1081893}         & 7{,}716  & 11{,}965{,}533  & 733 & Internet AS      \\
  \bottomrule
\end{tabular}
\end{table}
\subsection{Datasets}\label{sec:datasets}

We evaluate on \emph{two} task families. For dynamic \textbf{link prediction} we use the six public datasets for snapshot DGNNs~\citep{10.1145/3534678.3539300,10.1145/3580305.3599551,li2025tmetanet} (Table~\ref{tab:datasets}). \textbf{Bitcoin-OTC} and \textbf{Bitcoin-Alpha} are who-trusts-whom signed networks from two Bitcoin trading platforms. \textbf{UCI-Message} records private messages between students at UC Irvine. \textbf{Reddit-Title} and \textbf{Reddit-Body} capture hyperlink interactions between subreddit communities. \textbf{AS-733} consists of daily snapshots of the Internet autonomous-system topology. Following ROLAND~\citep{10.1145/3534678.3539300}, we discretize Bitcoin, UCI, and Reddit into weekly snapshots, and AS-733 into daily snapshots; the same snapshot frequencies are used by all methods.

For dynamic \textbf{node classification} we use the three continuous-time JODIE benchmark~\citep{10.1145/3292500.3330895} datasets (Wikipedia, Reddit, and MOOC), discretized into fixed-width snapshots, (Table~\ref{tab:nc_datasets}). \textbf{Wikipedia} records edits by users on Wikipedia pages over a one-month window; the binary label is whether the editing user is subsequently banned. \textbf{Reddit} records users' posts to subreddits over the same window; the label indicates whether the user is later banned from posting. \textbf{MOOC} records student interactions with online-course units; the label is whether the student subsequently drops the course. All three are released as continuous-time streams~\citep{10.1145/3292500.3330895}; we discretize them into $\Delta=6$~h snapshots.

\begin{table}[t]
\centering
\caption{Dynamic node-classification dataset statistics. The graphs are bipartite: $|\mathcal{V}|$ counts users and items, and $|\mathcal{E}|$ is the number of interaction events. $T_{\Delta}$ is the number of snapshots at bucket width $\Delta=6$~h; positive rate is the fraction of interaction events that carry a positive state-change label.}
\label{tab:nc_datasets}
\small
\setlength{\tabcolsep}{5pt}
\begin{tabular}{lrrrrl}
  \toprule
  Dataset      & $|\mathcal{V}|$ & $|\mathcal{E}|$ & $T_{\Delta}$ & pos.\ rate & Domain \\
  \midrule
  Wikipedia    & 9{,}227   & 157{,}474 & 124 & $0.14$\% & Edit logs \\
  Reddit       & 10{,}984  & 672{,}447 & 124 & $0.05$\% & Posts     \\
  MOOC         & 7{,}144   & 411{,}749 & 120 & $0.99$\% & E-learning \\
  \bottomrule
\end{tabular}
\end{table}
\subsection{Baselines}
\subsubsection{Link Prediction}
We compare against fourteen methods spanning every major DGNN family:
\begin{itemize}
  \item \textbf{Static.} GCN~\citep{kipf2017semisupervised} trained on the final snapshot.
  \item \textbf{Dynamic embeddings.} DynGEM~\citep{goyal2018dyngem}, dyngraph2vecAE and dyngraph2vecAERNN~\citep{GOYAL2020104816}.
  \item \textbf{Temporal-first.} EvolveGCN-H, EvolveGCN-O~\citep{egcn}.
  \item \textbf{Spatial-first.} GCRN-GRU, GCRN-LSTM and GCRN - Baseline \citep{10.1007/978-3-030-04167-0_33}; T-GCN~\citep{8809901}; the three ROLAND variants Moving-Average, MLP-Update and GRU-Update~\citep{10.1145/3534678.3539300}.
  \item \textbf{Meta-learning.} TMetaNet~\citep{li2025tmetanet}
\end{itemize}
All baseline numbers are taken verbatim from the original publications: the fixed-split numbers for GCN, DynGEM, dyngraph2vec and EvolveGCN are those reported by \citet{egcn}, the ROLAND fixed-split numbers and all live-update baseline numbers (including the OOM entries) are from Tables~2 and~3 of \citet{10.1145/3534678.3539300}, and the TMetaNet numbers are from Table~1 (live-update setting) of \citet{li2025tmetanet}. We verified that the published TMetaNet numbers are reproducible under our system configuration on every dataset except AS-733, and therefore report the published values. On AS-733, TMetaNet could not be run on our system, and we mark it as N/R (not reported; see Table~\ref{tab:live_update}). SiST-GNN and the persistence reference are also run by us under the same protocols; SiST-GNN is repeated with five random seeds, and we report the mean and standard deviation of MRR.

\paragraph{A zero-parameter persistence reference.}\label{sec:persistence}
MRR on these benchmarks is not, on its own, interpretable: a dynamic graph whose edges largely recur from one snapshot to the next can be scored well without any temporal modelling at all. Calibrating learned scores against a memorisation heuristic is established practice in the continuous-time literature: \citet{10.5555/3600270.3602656} introduce EdgeBank, a parameter-free baseline that predicts previously observed edges and is competitive with trained models on several benchmarks, and it is reported alongside learned methods in TGB~\citep{huang2023temporal}. Such a reference is, however, not routinely reported on the snapshot benchmarks used here, so we add one. \emph{Persistence} simply copies the current snapshot forward, scoring a candidate pair by whether it is already an edge,
\begin{equation}
  s_{\text{pers}}(u,v) \;=\; \mathbf{1}\!\left[(u,v) \in \mathcal{E}_t\right],
  \label{eq:persistence}
\end{equation}
and ranks the targets $\mathcal{E}_{t+1}$ with this score. It has no parameters and is never trained. We evaluate it through exactly the pipeline used for every learned model: the same snapshot sequence and discretisation, the same per-source MRR, the same $1{,}000$ filtered negative destinations per source, and the same set of scored snapshots. The only difference is the scoring function, so the comparison isolates what the learned models add over the trivial strategy of predicting that tomorrow looks like today.

\subsubsection{Node Classification}
We compare SiST-GNN against two families of dynamic-GNN methods.
\begin{itemize}
\item \textbf{CTDG based models.} They consume the native continuous-time event stream and are therefore the harder comparison for any snapshot-based approach: JODIE~\citep{10.1145/3292500.3330895}, TGN~\citep{tgn_icml_grl2020}, DyREP~\citep{trivedi2018dyrep}, TGAT~\citep{Xu2020Inductive}, and APAN~\citep{10.1145/3448016.3457564}.
\item \textbf{DTDG based models.} They operate on the same discretized snapshot sequence as SiST-GNN: EvolveGCN-H and EvolveGCN-O~\citep{egcn}, ROLAND~\citep{10.1145/3534678.3539300}, a DTDG-adapted TGN~\citep{tgn_icml_grl2020}, the Hawkes-process-based TREND~\citep{10.1145/3485447.3512164}, and the MLP-mixer GraphMixer~\citep{cong2023do}.
\end{itemize}
CTDG baseline numbers and the two EvolveGCN entries are taken verbatim from Table~6 of the survey of \citet{11202740}; the remaining DTDG baselines follow their original AUC protocol.
We do \emph{not} compare against pre-training or prompt-based methods~\citep{yu2025nodetime}, as they rely on a different upstream pre-training phase that is incompatible with our supervised pipeline; we revisit this gap in Section~\ref{sec:limitations}.

\subsection{Implementation}
SiST-GNN is implemented in PyTorch~\citep{10.5555/3454287.3455008} and PyTorch Geometric~\citep{DBLP:journals/corr/abs-1903-02428}. For the link prediction task, our default configuration is a 2-layer SiST-GNN with GCNConv as the spatial backbone, hidden dimension $d_h = 128$, learnable node embeddings of dimension $128$, dropout $0.1$, and the Adam optimiser~\citep{kingma2015adam} with learning rate $10^{-3}$ and weight decay $10^{-5}$. We train for $100$ epochs. The link-prediction head, loss, and negative-sampling protocol are as described in Section~\ref{sec:lp_head}.

For node classification task, our configuration is $\Delta=6$~h, $L=2$ layers, $d_h=256$, \textsc{GCNConv} backbone, \emph{balanced} BCE loss, $100$~epochs and early stopping with patience $5$, and the Adam optimiser with learning rate $10^{-3}$. The node-classification head and loss are as described in Section~\ref{sec:nc_head}.

The device configurations are Intel(R) Xeon(R) Gold 5120 CPU @ 2.20 GHz, 256 GB RAM, NVIDIA A100 40GB GPU. The source code is available at this \href{https://github.com/Roy-Shubhajit/SiST-GNN}{\faGithub\ Github Repository}.

\subsection{Link Prediction Results}
\subsubsection{Fixed-Split Results}

Table~\ref{tab:fixed_split} reports MRR on the three datasets used in ROLAND's fixed-split evaluation. SiST-GNN is the best learned method on all three, attaining $0.260$ on Bitcoin-OTC, $0.314$ on Bitcoin-Alpha, and $0.232$ on UCI-Message, and improving on the previous best (ROLAND GRU) by $+18.2\%$, $+8.7\%$, and $+1.3\%$ respectively.

The persistence reference explains why those three margins differ so much. On the two Bitcoin trust networks, copying the previous snapshot forward scores almost nothing ($0.021$ and $0.016$): consecutive weekly snapshots share hardly any edges, so essentially all of the reported MRR has to be produced by the model, and SiST-GNN produces more of it than any baseline. UCI-Message is the opposite regime. About a quarter of its source nodes repeat an edge from one snapshot to the next, persistence alone reaches $0.246$, and \emph{no} learned method in the table, ours included, matches it. We therefore report the $+1.3\%$ margin over ROLAND GRU on UCI-Message as a difference between two models that are both below the trivial strategy.
\subsubsection{Live-Update Results}

Table~\ref{tab:live_update} reports the live-update results across all six benchmarks. SiST-GNN is the strongest \emph{learned} method on five of the six, with TMetaNet retaining a narrow lead on Bitcoin-Alpha; persistence additionally exceeds every learned method on UCI-Message.
\begin{table}[t]
\centering
\caption{Live-update MRR. SiST-GNN: mean $\pm$ std.\ over 5 seeds; baseline numbers are reported verbatim from ROLAND (3 seeds) and TMetaNet (5 seeds). N/R = not reported. Best per column in \textbf{bold}; OOM = out-of-memory, as reported by ROLAND. The two improvement rows are relative to the strongest learned baseline and to persistence (Section~\ref{sec:persistence}) respectively.}
\label{tab:live_update}
\resizebox{\linewidth}{!}{%
\begin{tabular}{lrrrrrr}
  \toprule
                  & AS-733            & Reddit-Title         & Reddit-Body          & UCI-Message          & Bitcoin-OTC          & Bitcoin-Alpha         \\
  \midrule
  \multicolumn{7}{c}{\textit{Zero-parameter reference}} \\
  \midrule
  Persistence               & $0.823$           & $0.170$              & $0.151$              & $\mathbf{0.186}$     & $0.022$              & $0.021$               \\
  \midrule
  \multicolumn{7}{c}{\textit{Baselines with standard training (BPTT)}} \\
  \midrule
  EvolveGCN-H~\citep{egcn}     & OOM               & OOM                  & $0.148 \pm 0.013$    & $0.061 \pm 0.040$    & $0.067 \pm 0.035$    & $0.079 \pm 0.032$     \\
  EvolveGCN-O~\citep{egcn}     & OOM               & OOM                  & OOM                  & $0.071 \pm 0.009$    & $0.085 \pm 0.022$    & $0.071 \pm 0.025$     \\
  GCRN-GRU~\citep{10.1007/978-3-030-04167-0_33}          & OOM               & OOM                  & OOM                  & $0.080 \pm 0.012$    & OOM                  & OOM                   \\
  GCRN-LSTM~\citep{10.1007/978-3-030-04167-0_33}         & OOM               & OOM                  & OOM                  & $0.083 \pm 0.001$    & OOM                  & OOM                   \\
  GCRN-Baseline~\citep{10.1007/978-3-030-04167-0_33}     & OOM               & OOM                  & OOM                  & $0.069 \pm 0.004$    & $0.152 \pm 0.011$    & $0.141 \pm 0.005$     \\
  T-GCN~\citep{8809901}                  & OOM               & OOM                  & OOM                  & $0.054 \pm 0.024$    & $0.128 \pm 0.049$    & $0.088 \pm 0.038$     \\
  \midrule
  \multicolumn{7}{c}{\textit{Baselines with ROLAND incremental training}} \\
  \midrule
  EvolveGCN-H~\citep{egcn}     & $0.251 \pm 0.079$ & $0.165 \pm 0.026$    & $0.102 \pm 0.010$    & $0.057 \pm 0.012$    & $0.076 \pm 0.022$    & $0.054 \pm 0.015$     \\
  EvolveGCN-O~\citep{egcn}     & $0.163 \pm 0.002$ & $0.047 \pm 0.004$    & $0.033 \pm 0.001$    & $0.066 \pm 0.012$    & $0.032 \pm 0.004$    & $0.034 \pm 0.002$     \\
  GCRN-GRU ~\citep{10.1007/978-3-030-04167-0_33}          & $0.344 \pm 0.001$ & $0.338 \pm 0.006$    & $0.217 \pm 0.004$    & $0.089 \pm 0.004$    & $0.173 \pm 0.003$    & $0.140 \pm 0.004$     \\
  GCRN-LSTM ~\citep{10.1007/978-3-030-04167-0_33}         & $0.341 \pm 0.001$ & $0.344 \pm 0.005$    & $0.216 \pm 0.000$    & $0.091 \pm 0.010$    & $0.174 \pm 0.004$    & $0.146 \pm 0.005$     \\
  GCRN-Baseline ~\citep{10.1007/978-3-030-04167-0_33}     & $0.336 \pm 0.002$ & $0.351 \pm 0.001$    & $0.218 \pm 0.002$    & $0.095 \pm 0.008$    & $0.183 \pm 0.002$    & $0.145 \pm 0.003$     \\
  T-GCN ~\citep{8809901}                  & $0.343 \pm 0.002$ & $0.391 \pm 0.004$    & $0.251 \pm 0.001$    & $0.080 \pm 0.015$    & $0.083 \pm 0.011$    & $0.069 \pm 0.013$     \\
  \midrule
  \multicolumn{7}{c}{\textit{ROLAND variants, topology-augmented meta-learning, and SiST-GNN}} \\
  \midrule
  ROLAND Moving-Avg. ~\citep{10.1145/3534678.3539300}    & $0.309 \pm 0.011$ & $0.362 \pm 0.007$    & $0.289 \pm 0.038$    & $0.075 \pm 0.006$    & $0.120 \pm 0.002$    & $0.096 \pm 0.010$     \\
  ROLAND MLP-Update ~\citep{10.1145/3534678.3539300}     & $0.329 \pm 0.021$ & $0.395 \pm 0.006$    & $0.291 \pm 0.008$    & $0.103 \pm 0.010$    & $0.154 \pm 0.010$    & $0.148 \pm 0.012$     \\
  ROLAND GRU-Update ~\citep{10.1145/3534678.3539300}     & $0.340 \pm 0.001$ & $0.425 \pm 0.015$    & $0.362 \pm 0.002$    & $0.112 \pm 0.008$    & $0.194 \pm 0.004$    & $0.157 \pm 0.007$     \\
  TMetaNet ~\citep{li2025tmetanet}             & N/R$^{\dagger}$      & $0.427 \pm 0.010$  & $0.349 \pm 0.011$  & $0.110 \pm 0.009$  & $0.181 \pm 0.012$  & $\mathbf{0.177 \pm 0.006}$   \\
  \midrule
  \textbf{SiST-GNN (ours)} & $\mathbf{0.889 \pm 0.001}$ & $\mathbf{0.439 \pm 0.003}$    & $\mathbf{0.369 \pm 0.003}$    & $0.175 \pm 0.014$    & $\mathbf{0.196 \pm 0.007}$    & $0.174 \pm 0.013$     \\
  \midrule
  Impr.\ over best learned & $+158.4\%$    & $+2.8\%$            & $+1.9\%$            & $+56.3\%$            & $+1.0\%$            & $-1.7\%$            \\
  Impr.\ over persistence  & $+8.0\%$      & $+158.2\%$          & $+144.4\%$          & $-5.9\%$             & $+790.9\%$          & $+728.6\%$          \\
  \bottomrule
\end{tabular}%
}
\par\vspace{2pt}
{\raggedright\footnotesize $^{\dagger}$TMetaNet could not be run on AS-733 under our system configuration (256\,GB RAM, one NVIDIA A100 40\,GB GPU): on this graph it requires more than 256\,GB of RAM and fails with an out-of-memory error, and its estimated running time is on the order of years. \citet{li2025tmetanet} do not report results on AS-733.\par}
\end{table}
Against the strongest baseline, SiST-GNN improves on five of the six benchmarks, by $+158.4\%$ on AS-733 and $+56.3\%$ on UCI-Message, modestly on the Reddit hyperlink graphs ($+2.8\%$ and $+1.9\%$) and on Bitcoin-OTC ($+1.0\%$), and by $-1.7\%$ on Bitcoin-Alpha, where TMetaNet remains ahead.

Considering persistence, the same six numbers tell a different and more informative story, and the two large margins shrink sharply. AS-733 is a near-static graph: persistence alone scores $0.823$, so the headline $+158.4\%$ over the strongest learned baseline is really $+8.0\%$ over copying the previous snapshot, and most of what separates the learned methods there is how well they avoid discarding structure that was already given to them. On UCI-Message the $+56.3\%$ margin inverts outright, to $-5.9\%$: persistence is the best method on that column. The benchmarks where SiST-GNN adds the most over the trivial strategy are exactly the two where that strategy fails, Bitcoin-OTC ($+791\%$) and Bitcoin-Alpha ($+729\%$), followed by the two Reddit graphs ($+158\%$ and $+144\%$). This is the reverse of the ordering suggested by the baseline column, and we think it is the ordering that should be used to judge the construction: cross-time messages are worth most when a neighbor's past is not simply a copy of its present.

BPTT-trained baselines (EvolveGCN, GCRN, T-GCN) fail with OOM on AS-733 and Reddit benchmarks. SiST-GNN, trained under ROLAND's incremental scheme, runs on all six datasets on a single GPU. Memory growth is constant in $T$ because hidden states are detached between snapshots, and the augmented graph at snapshot $t$ adds only $|\mathcal{E}_t|+N$ extra edges.
    \begin{table}[t]
    \centering
    \caption{Fixed-split MRR. SiST-GNN: mean $\pm$ std.\ over 5 seeds; baseline numbers are reported verbatim from EvolveGCN and ROLAND (ROLAND reports mean $\pm$ standard error over 3 seeds). Best per column in \textbf{bold}. \emph{Persistence} (Section~\ref{sec:persistence}) is excluded from the ``best learned'' row.}
    \label{tab:fixed_split}
    \begin{tabular}{lrrr}
  \toprule
                              & Bitcoin-OTC          & Bitcoin-Alpha        & UCI-Message          \\
  \midrule
  \multicolumn{4}{c}{\textit{Zero-parameter reference}} \\
  \midrule
  Persistence                 & $0.021$             & $0.016$             & $\mathbf{0.246}$   \\
  \midrule
  \multicolumn{4}{c}{\textit{Learned methods}} \\
  \midrule
  GCN~\citep{kipf2017semisupervised}                          & $0.0025$            & $0.0031$            & $0.1141$            \\
  DynGEM~\citep{goyal2018dyngem}                    & $0.0921$            & $0.1287$            & $0.1055$            \\
  dyngraph2vecAE~\citep{GOYAL2020104816}      & $0.0916$            & $0.1478$            & $0.0540$            \\
  dyngraph2vecAERNN~\citep{GOYAL2020104816}   & $0.1268$            & $0.1945$            & $0.0713$            \\
  EvolveGCN-H~\citep{egcn}           & $0.0690$            & $0.1104$            & $0.0899$            \\
  EvolveGCN-O~\citep{egcn}           & $0.0968$            & $0.1185$            & $0.1379$            \\
  \midrule
  ROLAND Moving-Avg.~\citep{10.1145/3534678.3539300}          & $0.047\pm 0.002$ & $0.140\pm0.011$ & $0.065\pm0.005$ \\
  ROLAND MLP~\citep{10.1145/3534678.3539300}                  & $0.078\pm0.002$ & $0.156 \pm0.011$ & $0.088 \pm 0.011$ \\
  ROLAND GRU~\citep{10.1145/3534678.3539300}                  & $0.220 \pm 0.017$ & $0.289\pm 0.012$ & $0.229\pm 0.062$ \\
  \midrule
  \textbf{SiST-GNN (ours)}    & $\mathbf{0.260 \pm 0.019}$    & $\mathbf{ 0.314 \pm 0.015}$    & $0.232 \pm 0.006$    \\
  \midrule
  Impr.\ over best learned    & $+18.2\%$           & $+8.7\%$           & $+1.3\%$            \\
  Impr.\ over persistence     & $+1138\%$          & $+1863\%$          & $-5.7\%$           \\
  \bottomrule
\end{tabular}
\end{table}

\begin{table}[ht]
\centering
\caption{Dynamic node classification test AUC (\%, mean$\,\pm\,$std for SiST-GNN over 25 random seeds). Best DTDG-based result per column in \textbf{bold}.}
\label{tab:nc_results}
\setlength{\tabcolsep}{5pt}
\small
\begin{tabular}{lrrr}
\toprule
 & Wikipedia & Reddit & MOOC \\
\midrule
\multicolumn{4}{c}{\textit{CTDG based models}} \\
\midrule
JODIE~\citep{10.1145/3292500.3330895}           & $80.21$ & $63.92$ & $60.99$ \\
TGN~\citep{tgn_icml_grl2020}                    & $86.85$ & $68.21$ & $53.86$ \\
DyREP~\citep{trivedi2018dyrep}                   & $83.57$ & $53.10$ & $65.61$ \\
TGAT~\citep{Xu2020Inductive}                    & $85.71$ & $67.24$ & $63.11$ \\
APAN~\citep{10.1145/3448016.3457564}            & $87.01$ & $55.90$ & $64.12$ \\
\midrule
\multicolumn{4}{c}{\textit{DTDG based models}} \\
\midrule
EvolveGCN-H~\citep{egcn}     & $69.81$ & $59.21$ & $67.92$ \\
EvolveGCN-O~\citep{egcn}     & $79.20$ & $59.64$ & $67.82$ \\
ROLAND~\citep{10.1145/3534678.3539300}          & $58.86 \pm 10.30$ & $48.25 \pm 9.57$ & $49.93 \pm 6.74$ \\
TGN~\citep{tgn_icml_grl2020}                    & $50.61 \pm 13.60$ & $49.54 \pm 6.23$ & $50.33 \pm 4.47$ \\
TREND~\citep{10.1145/3485447.3512164}           & $69.92 \pm 9.27$ & $64.85 \pm 4.71$ & $66.79 \pm 5.44$ \\
GraphMixer~\citep{cong2023do}                   & $65.43 \pm 4.21$ & $60.21 \pm 5.36$ & $63.72 \pm 4.98$ \\
\midrule
\textbf{SiST-GNN (ours)}                       & $\mathbf{84.76 \pm 3.05}$ & $\mathbf{72.27 \pm 1.88}$ & $\mathbf{83.32 \pm 0.29}$ \\
\midrule
Improvement (DTDG)             & $+7.02\%$ & $+11.44\%$ & $+22.67\%$ \\
\bottomrule
\end{tabular}
\end{table}
\subsection{Node Classification Results}\label{sec:nc_results}

Table~\ref{tab:nc_results} summarises the results on the node classification task. Among DTDG-based models, SiST-GNN achieves the best AUC on every dataset. It also remains competitive with the CTDG-based model family, which operates on the \emph{native} continuous-time event stream, while SiST-GNN sees only the discretized snapshot sequence. 

CTDG architectures (JODIE, TGN, DyREP, TGAT, APAN) consume the event ordering directly and are expected to be hard to match on a continuous-time benchmark. Yet SiST-GNN, fed only a snapshot sequence with $\Delta=6$~h buckets, achieves comparable results to CTDG models on every dataset.
Closing the CTDG--DTDG gap without any continuous-time machinery is the strongest single piece of evidence that the simultaneous-fusion construction, not the temporal interface, drives the improvement.

Restricted to the discrete-time competitors, SiST-GNN improves test AUC by $+7.0\%$ on Wikipedia (over EvolveGCN-O), $+11.4\%$ on Reddit (over TREND), and $+22.7\%$ on MOOC (over EvolveGCN-H), a relative gain of $7-23\%$. The Wikipedia margin is particularly noteworthy because every DTDG baseline collapses to $51$--$79$\% on this dataset, suggesting that the snapshot interface alone is not the bottleneck; the bottleneck is the strict spatial/temporal ordering that SiST-GNN removes.

\begin{figure*}[t]
\centering
\includegraphics[width=\textwidth]{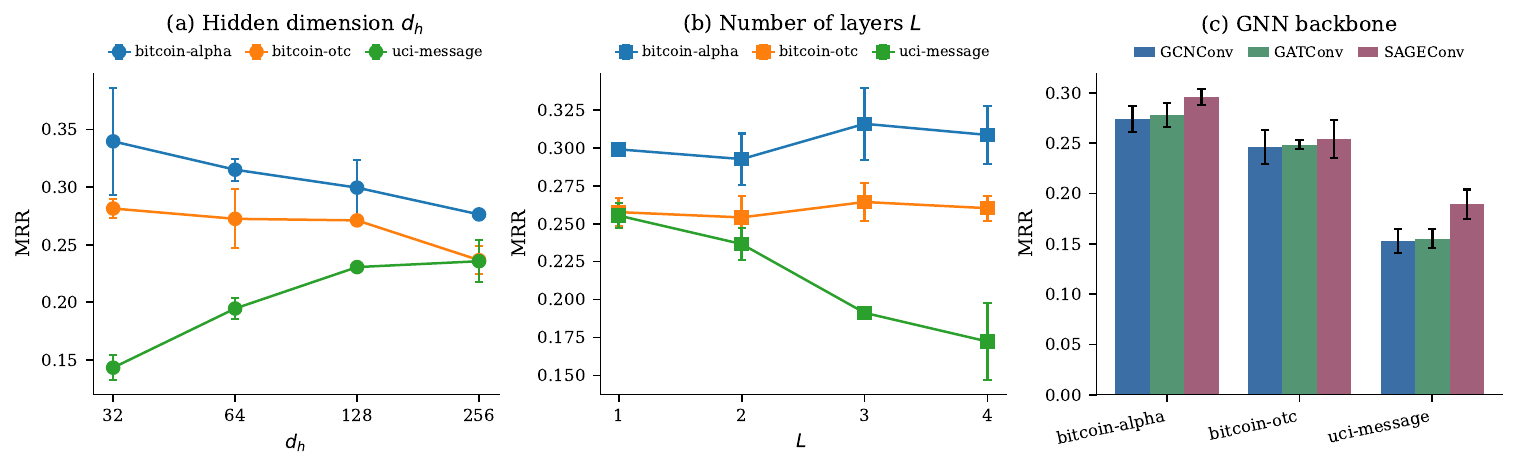}
\caption{Sensitivity and architectural ablations on the three fixed-split datasets. \textbf{(a)}~MRR vs.\ hidden dimension $d_h \in \{32, 64, 128, 256\}$. \textbf{(b)}~MRR vs.\ number of stacked layers $L\in\{1,2,3,4\}$. \textbf{(c)}~Spatial backbone (GCN / GAT / SAGE), with edge attributes disabled so that all three receive identical inputs. Markers and bars show the mean over 3 seeds; error bars are one standard deviation.}
\label{fig:ablations}
\end{figure*}
\begin{figure*}[t]
\centering
\includegraphics[width=\textwidth]{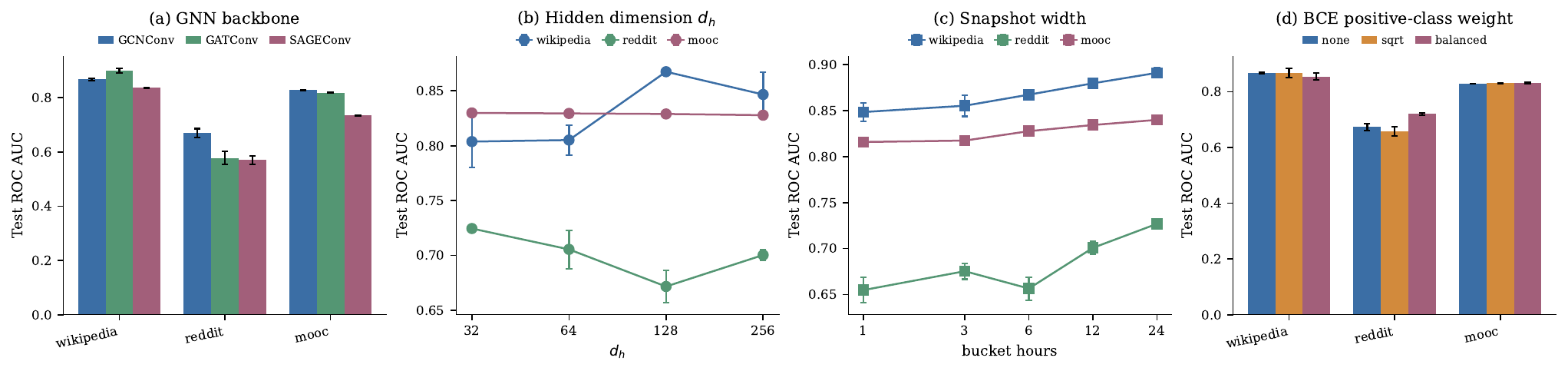}
\caption{Sensitivity ablations for dynamic node classification (Wikipedia, Reddit, MOOC). \textbf{(a)}~Spatial backbone (GCN / SAGE / GAT). \textbf{(b)}~Hidden dimension $d_h\in\{32,64,128,256\}$. \textbf{(c)}~AUC vs.\ snapshot width $\Delta\in\{1,3,6,12,24\}$~h. \textbf{(d)}~Positive-class weighting strategy: \emph{none} (plain BCE), \emph{sqrt}, \emph{balanced}. Markers/bars are the mean of $25$ seeds; error bars are one standard deviation.}
\label{fig:ablations_nc}
\end{figure*}
\subsection{Ablations and Sensitivity}
\paragraph{Sensitivity to hidden dimension.}
We sweep $d_h \in \{32, 64, 128, 256\}$ in Figure~\ref{fig:ablations}(a) and Figure~\ref{fig:ablations_nc}(b). The optimal hidden dimension is dataset-dependent, and on link prediction the two families move in opposite directions. Both Bitcoin trust graphs degrade monotonically as the model widens, from $0.281$ to $0.237$ on Bitcoin-OTC and from $0.340$ to $0.276$ on Bitcoin-Alpha, and are best at the smallest setting $d_h=32$; these snapshots carry only $150$--$190$ edges over several thousand nodes, so additional width is spent on overfitting. UCI-Message shows the opposite trend, climbing monotonically from $0.143$ at $d_h=32$ to $0.236$ at $d_h=256$. The node-classification sweep is similarly heterogeneous; Wikipedia favors $d_h=128$, Reddit is best at the smallest setting, and MOOC is essentially flat across the sweep. We keep $d_h=128$ as a single common configuration for link prediction in Table \ref{tab:fixed_split} and \ref{tab:live_update}, and $d_h=256$ for node classification in Table \ref{tab:nc_results}; it is not the optimum on any individual dataset, and per-dataset tuning would raise the trust-network numbers further.

\paragraph{Sensitivity to depth.}
We vary $L \in \{1, 2, 3, 4\}$ in Figure~\ref{fig:ablations}(b). Depth matters far less than width, and again the effect is dataset-dependent. On both Bitcoin graphs the curve is flat: the full range spans $0.254$--$0.264$ on Bitcoin-OTC and $0.293$--$0.316$ on Bitcoin-Alpha, comparable to the seed-level spread, so no setting is separable from any other. UCI-Message instead degrades monotonically and substantially, from $0.255$ at $L=1$ to $0.172$ at $L=4$, a $33\%$ drop. This is the expected over-smoothing signature on its denser snapshots ($\approx 1{,}300$ edges over $1{,}899$ nodes), where three or four hops already reach most of the graph and node representations converge. We default to $L=2$ in Table \ref{tab:fixed_split}, \ref{tab:live_update}; a single layer would be preferable on UCI-Message alone.

\paragraph{GNN-backbone robustness.}
Figure~\ref{fig:ablations}(c) and Figure~\ref{fig:ablations_nc}(a). On link prediction, GraphSAGE~\citep{10.5555/3294771.3294869} is the strongest backbone on all three benchmarks, with the clearest margin on UCI-Message ($0.189$ against $0.153$ for GCN), while GCNConv and GAT~\citep{velickovic2018graph} are indistinguishable from each other everywhere. So that the three operators receive identical inputs, this panel disables edge attributes: GCNConv consumes them as edge weights whereas GraphSAGE ignores them, and leaving them enabled would compare the fallback path rather than the aggregator. The picture inverts on node classification, where GCNConv dominates Reddit and MOOC and GAT is best on Wikipedia. The augmented-graph construction is therefore agnostic to the specific aggregator, though the choice matters at the margin and tracks the usual backbone--dataset interactions reported in the static-GNN literature.

\paragraph{Snapshot width $\Delta$ and class weighting.}
Figure~\ref{fig:ablations_nc}(c) sweeps the discretisation width
$\Delta \in \{1, 3, 6, 12, 24\}$\,h. AUC rises \emph{monotonically} with $\Delta$ on all three benchmarks. The Reddit effect is consistent with its very low positive rate ($0.05$\%): wider buckets concentrate more rare positive events into each snapshot, giving the LSTM a less noisy temporal signal and the GNN a denser graph to aggregate over. Figure~\ref{fig:ablations_nc}(d) is the complementary view through class weighting. Wikipedia and MOOC are insensitive to the BCE positive-class weight (within $\pm 0.015$ AUC across \texttt{none}, \texttt{sqrt}, \texttt{balanced}), whereas Reddit gains $+4.7$ points moving from unweighted ($0.673$) to balanced ($0.720$). Discretization and reweighting are thus two complementary handles for severe class imbalance.

\section{Analysis and Discussion}\label{sec:analysis}

The sequential paradigms create an information bottleneck: a spatial-first model cannot condition its message-passing aggregator on temporal context, since none has been computed yet; a temporal-first model cannot condition its recurrent cell on the current neighborhood, since none has been observed yet. SiST-GNN dissolves this bottleneck. Consider a node $u$ with neighbors $v_1, v_2$: suppose $v_1$ has been consistently active (rich temporal signal, perhaps noisy current features) while $v_2$ has just appeared in the graph (clean current features, no history). A spatial-first model must use the structural representation of both before any temporal feature is seen; a temporal-first model must summarise both temporally before aggregation. SiST-GNN exposes the GNN to $v_1$'s history \emph{and} $v_2$'s current features on the same layer, letting the aggregation weights pick the right blend per-neighbor. This is the benefit of the doubled message pool described in Section~\ref{sec:theory}, and the source of the strict generalization result of Theorem~\ref{thm:subsume}.

The persistence reference sharpens this picture, and partly reverses it. Measured against the strongest published baseline, the largest gains are on AS-733 ($+158.4\%$) and UCI-Message ($+56.3\%$); measured against copying the previous snapshot forward, those two become $+8.0\%$ and $-5.9\%$, while the Bitcoin trust networks move from the smallest margins to by far the largest ($+791\%$ and $+729\%$). The two views disagree because they answer different questions. A large margin over prior work on AS-733 mostly reflects that earlier methods left recurring edges on the table; a large margin over persistence on Bitcoin reflects genuine prediction, since there is nothing to copy.

Both views are consistent with the same mechanism. Cross-time messages help in proportion to how much a neighbour's past differs from its present while still carrying signal. On the Bitcoin trust networks, consecutive weekly snapshots share almost no edges, so the current snapshot is nearly uninformative about the next one and the temporal trajectory is the only usable signal; this is where simultaneous fusion contributes most. On AS-733 the current snapshot is already almost the answer, so every reasonable method scores highly and the differences between them are small in absolute terms. UCI-Message sits in between and is the one case where the construction does not pay: enough structure recurs that copying wins, but not enough that the learned models can exploit it better than copying does.

Summarising, the gap between SiST-GNN and ROLAND-GRU is largest when the graph has many recurring edges, because recurrence is precisely what a spatial-first model discards when it compresses each snapshot before looking at time. The gap over \emph{persistence} behaves in the opposite way, and is largest when recurrence is low and a neighbour's trajectory is the only signal available. Bitcoin-Alpha remains the one benchmark where a published method is ahead of us, by $1.7\%$, which is within the seed-level spread of both.

\subsection{Limitations}\label{sec:limitations}
Discretizing a CTDG into buckets discards the within-bucket event order, and the results depend on the bucket width: in the $\Delta$-sweep of Figure~\ref{fig:ablations_nc}(c), AUC rises monotonically with $\Delta$ on all three datasets. We attribute this to class imbalance rather than to the lost ordering: when the positive rate is extreme, fine-grained buckets contain too few positive events to provide a useful supervised signal. Our experiments therefore do not isolate how much the discarded within-bucket ordering itself costs. A complementary fix is to reweight the BCE loss, the \texttt{balanced} setting recovers most of the gap at fixed $\Delta$ (Figure~\ref{fig:ablations_nc}(d)). Beyond discretization, \citet{yu2025nodetime} pre-train a dynamic-GNN encoder on auxiliary tasks and then adapt it with node- and time-conditional prompts. Because that family relies on an upstream pre-training phase, head-to-head numbers are not directly comparable to our supervised pipeline. Designing a pre-training protocol for SiST-GNN and benchmarking it against prompt-tuned baselines on the same downstream task is a future direction. Beyond link prediction and node classification, our architecture has not yet been evaluated on dynamic benchmarks such as TGB~\citep{huang2023temporal}, on multi-task regimes, or on graph-level prediction.

\section{Conclusion}\label{sec:conclusion}

We identified a structural limitation shared by virtually all snapshot-based dynamic GNNs, the strict ordering between spatial and temporal computation, and showed that fusing the two modalities \emph{within} a single message-passing step is effective. The proposed SiST-GNN architecture doubles the per-layer distinguishable message pool, generalizes both the spatial-first and the temporal-first paradigm via two-layer composition, is permutation-equivariant, and incurs only a constant-factor overhead in per-snapshot compute. Empirically, SiST-GNN is the strongest learned method on the standard public dynamic link prediction benchmarks under the fixed-split protocol, with relative MRR gains of $1.3-18.2\%$ over the prior best, and leads on five of six datasets under live-update. Comparing against a zero-parameter persistence reference under the same protocol shows where those gains are substantive: they are largest, by a wide margin, on the Bitcoin trust networks, whose snapshots share almost no edges, and they vanish on UCI-Message, where copying the previous snapshot outperforms every learned method. We recommend that future work on these benchmarks report persistence alongside MRR. The same architecture also transfers to dynamic node classification. SiST-GNN improves the discrete-time baseline and achieves results comparable to continuous-time baselines despite the snapshot-discretization handicap. Beyond the empirical gains, instantiating the two views of a node as neighbors in a temporally augmented graph is a general construction that may be useful wherever a model must combine evolving and structural information in a single representation. Several directions remain open. First, the current augmentation materializes a cross-time edge for every node in consecutive snapshots, and learning a sparse mask over these cross-time edges would allow the construction to scale to TGB-size graphs~\citep{huang2023temporal}. Second, extending the construction to continuous-time event streams natively would remove the snapshot-discretization step required for the JODIE node-classification benchmarks and recover the fine-grained event ordering that the binning procedure discards; it would also remove the sensitivity to the bucket width $\Delta$ that our sweep exposes. Third, although the present evaluation centers on dynamic link prediction and dynamic node classification, the temporally augmented graph is task-agnostic, and applying SiST-GNN to multi-task regimes and to graph-level prediction is a natural next step. Finally, designing a pre-training protocol for SiST-GNN and benchmarking it against prompt-tuned baselines~\citep{yu2025nodetime} under a unified downstream protocol would close the remaining methodological gap with the recent pre-training and prompt-tuning line of work.

\bibliography{main}
\bibliographystyle{tmlr}

\end{document}